\documentclass[letterpaper, 10 pt, conference]{ieeeconf}  

\usepackage{tikz}
\usepackage{latexsym}
\usepackage{graphicx}
\usepackage{epstopdf}
\usepackage[T1]{fontenc}
\usepackage{enumerate}
\usepackage{graphicx}
\usepackage{mathrsfs}  
\usepackage{algorithm}
\usepackage{algpseudocode}
\usepackage{calrsfs}
\DeclareMathAlphabet{\pazocal}{OMS}{zplm}{m}{n}

\usepackage{array}
\usepackage{multirow}

\newcommand{\vct}{\mathbf}
\newcommand{\vect}[1]{\boldsymbol{#1}}
\usepackage{amsmath,amssymb}
\usepackage{xfrac}
\usepackage{upgreek}
\usepackage{eucal}
\usepackage{xcolor}
\usepackage{cite}
\usepackage{tabstackengine}

\usepackage{colortbl}
\usepackage{courier}
\usepackage{tikz}

\usetikzlibrary{calc}

\newtheorem{remark}{Remark}
\newtheorem{theorem}{Theorem}

\def\particulartemplate#1{
	\begin{tikzpicture}[overlay, remember picture]
	\draw let \p1 = (current page.west), \p2 = (current page.east) in
	node[minimum width=\x2-\x1, minimum height=0.1cm, rectangle, fill=yellow!35!white, anchor=north west, align=center, text width=\x2-\x1] at ($(current page.north west) + (0,-0.3)$) {\large \textbf{\texttt{#1}} };
	\end{tikzpicture}
}

\IEEEoverridecommandlockouts                              

\title{\LARGE \bf
Two-layer  adaptive trajectory tracking controller for quadruped robots on slippery terrains  
}

\author{ Despina-Ekaterini Argiropoulos$^{1,3}$, Dimitrios Papageorgiou$^{1,2}$, Michael Maravgakis$^{1,3}$, \\ Drosakis Drosakis$^1$ and Panos Trahanias$^{1,3}$
\thanks{The research leading to these results has received partial funding from the European Community’s HORIZON.1.2 - Marie Skłodowska-Curie Actions (MSCA) under Grant agreement No. 101072634, project RAICAM.}
\thanks{${}^1$Authors are with the Institute of Computer Science, Foundation for Research and Technology–Hellas, Heraklion 700 13, Greece.
        {\tt\small  \{dimpapag, despinar, maravgakis, drosakis, trahania\}@ics.forth.gr}}%
\thanks{${}^2$Dimitrios Papageorgiou is also with the Department of Electrical and Computer Engineering, Hellenic Mediterranean University, 714 10, Heraklion, Crete,  Greece.
      }%
\thanks{${}^3$Despina-Ekaterini Argiropoulos,  Michael Maravgakis and Panos Trahanias are also with the Department of Computer Science, University of Crete,
714 09 Heraklion, Crete, Greece.
       }%
}

\begin{document}

\maketitle
\particulartemplate{
 	This  paper  is  a  pre-print  version under review. 
}
\thispagestyle{empty}
\pagestyle{empty}

\begin{abstract}
Task space trajectory tracking for quadruped robots plays a crucial role  on  achieving dexterous maneuvers in unstructured environments. To fulfill the control objective, the robot should apply forces through the contact of the legs with the supporting surface, while maintaining its stability and controllability. In order to ensure the operation of the robot under these conditions, one has to account for the possibility of  unstable contact of the legs that arises when the robot operates on partially or globally slippery terrains. In this work, we propose an adaptive trajectory tracking controller for quadruped robots, which involves two prioritized  layers of adaptation for avoiding possible slippage of one or multiple legs. The adaptive framework is evaluated through simulations and validated through experiments.                
\end{abstract}

\section{Introduction}

One of the main advantages of legged robots 
is their capability to transverse unstructured environments, such as sewers or construction sites, which may involve a variety of challenging terrain types. This capability enables the utilization of legged robots in  applications potentially dangerous for the human, such as search and rescue operations or missions for inspection and maintenance in critical asset facilities. Except their structural complexity, such difficult-to-transverse environments also impose  dynamic challenges,  with the most dominant being the variation of the friction coefficient of the terrain. Partially or globally slippery terrains are considered to be one of the most frequent problems faced by legged robots, which may arise in case of mud, wet floor, oil or ice \cite{Jenelten2019}.  Slippage of any leg with respect to the supporting surface could trigger unknown and unmodelled dynamics  which would in-turn worsen  the trajectory tracking performance or even lead to robot's instability; e.g. it could lead to singular configurations or to configurations in which the contact with the supporting surface is lost.

Slippage is occurred when the applied force is not within  the static friction cone, defined by a threshold to the ratio between the tangential and the orthogonal force magnitudes; the so-called static friction coefficient. This coefficient, however, is difficult to model, varies in space and therefore, in most of the cases, it is considered to be unknown a priori. To tackle the problem of identifying the slippage phenomenon, current methods in literature propose the utilization of machine learning \cite{Camurri2017, Lin2021, Piperakis2022},
classical estimators such as  Extended/Unscented Kalman Filters \cite{Bloesch2012, Bloesch2013},  model-based estimators  \cite{Jemin2016}, most of the times   utilizing  proprioceptive information to yield an estimate of the probability of stable contact, as proposed in our recent work~\cite{Maravgakis2023}.    

After identifying the slippage of a single or multiple legs, one has to define a reactive behavior for ensuring the stability and controllability of the system. Common practices involve the utilization of estimators for slippage alongside with an  on-line trajectory generation mechanism for slippage recovery. For instance, in \cite{Fabian2019}, to recover from the slippage of one leg, the slipping leg is commanded to move towards its initial configuration, which is executed within a layer of a hierarchical framework that also involve the task-space trajectory. That means that the system may sacrifice the task performance for recovering from slippage, as it is not guaranteed that the joint-space recovery motion for the legs will not affect the overall tracking performance; e.g. in case of globally slippery terrains. In \cite{Focchi2018}, the slippage recovery strategy is handled also by a hierarchical framework which ensures that the reactive ground forces stay within an estimate of the friction cone. However, in case of infeasibility of solution (e.g. when all four legs slip) the motion is stopped (frozen state).  In \cite{Gerardo2018}, a Kalman Filter is utilized in a finite state machine framework to identify slippage and react. However, the paper focuses mainly on the  generalized-momentum disturbance observation, while the reactive behavior refers to the high-level control without considering the tracking of a specific trajectory.  From a machine learning perspective, some works employ the notion of reinforcement learning \cite{Joonho2020, Gwanghyeon2022} for tackling the problem of slippage during locomotion. Most specifically, in \cite{Joonho2020} reinforcement learning is employed based on proprioceptive sensorial information, while  in \cite{Gwanghyeon2022} the concurrent training of a state estimator and a reactive policy is considered; both works consider the training of the system in a simulated environment, exploring as much as possible the variety of  cases that the robot may face during its operation. 

Trajectory tracking can be utilized as the barebone for locomotion and facilitate walking over unstructured and slippery terrains. To the best of our knowledge, although some works  tackle the problem of trajectory tracking for quadruped robots in general, such as   \cite{Zihang2021} in joint space and \cite{Yulong2022} in task-space, the control problem is not explicitly  addressed considering the operation on a partially or globally slippery terrain.  On the other hand, works that make such a consideration (e.g. \cite{Fabian2019, Focchi2018, Gerardo2018, Joonho2020, Gwanghyeon2022}), are mainly tackling the locomotion problem and they cannot guarantee that the task-space control will remain unaffected, while most of them require the extensive training of the system in a simulated environment.


In this work, a novel trajectory tracking control scheme is proposed for quadruped robots, incorporating two prioritized layers of adaptation for minimizing the possible slippage of one or multiple legs. The proposed control framework builds upon our previous work on contact state estimation, in which the probability of stable contact for each foot is estimated independently based solely on proprioceptive data, namely IMU sensors attached to the robot's feet.
In the proposed control framework, the estimated probability is utilized as an input for two layers of adaptation. In particular, in the first layer, the adaptation of the distribution of the control effort among the legs is considered, without affecting the task performance, exploiting the redundancy of the quadruped robot. In the second layer, which is enabled only if the problem cannot be solved by the first layer, the time-scaling of the trajectory is considered, which affects  only the temporal properties of the task, without distorting the path followed by the robot. In other words, when the control effort distribution is not enough for tackling the problem of minimizing the slippage, the temporal scaling  of the trajectory  (slowing down the motion) will lead to an overall reduction of the control effort that has to be applied by the robot's legs. 

The main contributions of the proposed method are:
\begin{itemize}
\item A task-space trajectory tracking controller is proposed, which is theoretically proven to yield an asymptotically stable behavior. 
\item The proposed control scheme involves two layers of adaptation, as it accounts for partially or globally slippery terrains. Simulations show the ability of the proposed scheme to maintain stability and controllability of the system. 
\item The first layer of adaptation exploits the redundancy of the quadruped robot and dynamically distributes the control effort to the legs, based on the estimated slippage probability of each leg, while the second layer of adaptation involves the time-scaling of the trajectory and it is enabled only if the dynamic distribution cannot solve the problem of avoiding the slippage.  
\end{itemize}
Finally, to facilitate and promote research, the control scheme has been released as an open-source module in ROS/C++ coined as \textit{Maestro}~\cite{maestro} while the probabilistic contact estimator can be found in~\cite{PCE}. 


\section{Problem formulation and concept solution}

Consider the quadruped robot depicted in Fig.\ref{fig:go1}, having $n\in\mathbb{N}$ joints in each leg and let $q_{i,j}\in\mathbb{R}, i=1,...,4, j=1,...,n$ be the joint position variables of the $i$-th leg. Let $\vct{q}\triangleq[q_{1,1} \; q_{1,2} \; ... \; q_{4,n-1} \; q_{4,n}]^\intercal\in\mathbb{R}^{4n}$ be the vector of the total joint variables of the robot. Furthermore, let  $\{C\}$ be the frame placed at the Center of Mass (CoM) of the robot (as depicted in Fig.\ref{fig:go1}) and  ${}^c\vct{p}_{i}(q_{i,1},...,q_{i,n})\in\mathbb{R}^3$ be the position of the tip of each leg with respect to  $\{C\}$. The  position and orientation  of  $\{C\}$ with respect to the world frame $\{0\}$ is denoted by $\vct{p}_c\in\mathbb{R}^3$ and $\vct{R}_c\in SO(3)$ respectively. World frame $\{0\}$ could refer to a known inertial frame, or (in most of the cases) to the initial pose of the robot (i.e. $\{C\}$ at $t=0$). 
When all four tips are in contact with the supporting surface the mapping between the forces  $\vct{f}_i\in\mathbb{R}^3, i=1,...,4$  applied to the tips of the legs (e.g. reactive forces from the supporting surface) and the corresponding generalized force  $\vct{F}_c\triangleq[\vct{f}_c^\intercal \; \vect{\tau}_c^\intercal]^\intercal\in\mathbb{R}^6$ at the CoM, with $\vct{f}_c\in\mathbb{R}^3$ and  $\vect{\tau}_c\in\mathbb{R}^3$ being the force and torque at the CoM respectively, is the following:

\begin{equation}\label{eq:F_c}
\vct{F}_c = \vct{G}(\vct{q})\vct{F}_a,
\end{equation}
where 
\begin{equation} \label{eq:G}
\vct{G}(\vct{q})\triangleq 
    \begin{bmatrix}
\vct{I}_3 & \vct{I}_3 & \vct{I}_3 & \vct{I}_3\\
\vct{S}(\vct{p}_{c1}) & \vct{S}(\vct{p}_{c2}) & \vct{S}(\vct{p}_{c3}) & \vct{S}(\vct{p}_{c4})
\end{bmatrix}
\end{equation}
 and
\begin{equation}
\vct{F}_a\triangleq 
    \begin{bmatrix}
\vct{f}_1\\ \vct{f}_2\\ \vct{f}_3\\ \vct{f}_4    
\end{bmatrix}\in\mathbb{R}^{12},
\end{equation}
with $\vct{p}_{ci}(q_{i,1},...,q_{i,n})\triangleq\vct{R}_c{}^c\vct{p}_i(q_{i,1},...,q_{i,n}), i=1,...,4$, $\vct{I}_3\in\mathbb{R}^{3\times3}$  the identity matrix and $\vct{S}(.):\mathbb{R}^3\rightarrow\mathbb{R}^{3\times3}$ the skew symmetric mapping. Notice that $\vct{G}(\vct{q})$ belongs to $\mathbb{R}^{6\times12}$, and therefore 
the problem of solving \eqref{eq:F_c} with respect to $\vct{F}_a$, i.e. finding $\vct{F}_a$ for a given $\vct{F}_c$, is redundant.  
The mapping between the force applied to the tip of the $i$-th leg and the corresponding torques at the joints of the leg, is given by:
\begin{equation}\label{eq:joint_jacobian}
\vect{\tau}_i = \left(\vct{R}_c {}^c\vct{J}_{i}(q_{i,1},...,q_{i,n})\right)^\intercal \vct{f}_i,
\end{equation}
where ${}^c\vct{J}_{i}(q_{i,1},...,q_{i,n})\in\mathbb{R}^{3\times n}$ is the position part of the Jacobian of the leg with respect to the CoM and $\vect{\tau}_i\in\mathbb{R}^n$ the torques at the joints of the leg. 

 \begin{figure}[h!]
	\centering
		\includegraphics[width=.45\textwidth]{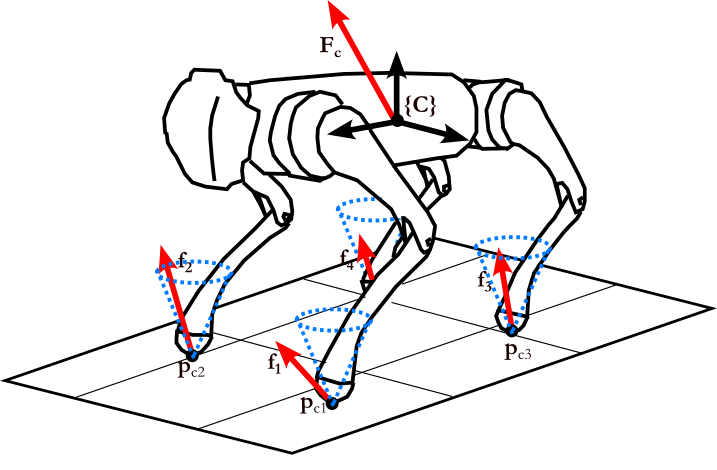}
	    \caption{Force distribution among the legs of the quadruped robot.}
	    \label{fig:go1}
\end{figure}

\begin{remark}
When less than four tips are in contact with the environment, $\vct{G}(\vct{q})$ has to be modified accordingly to involve only the legs that are in contact with the supporting surface. A representative example is in case of locomotion, in which one or more legs are in swing phase. 
\end{remark}

The dynamic model of the system, assuming that the inertia of the legs is negligible as compared to the inertia of the rest of the body, is given by:
\begin{equation}\label{eq:dyn_model}
\vct{H}_c\dot{\vct{V}}_c +\vct{C}_c\vct{V}_c + \vct{g}_c = \vct{F}_c, 
\end{equation}
where $\vct{V}_c\triangleq[\dot{\vct{p}}_c^\intercal \; \vect{\omega}_c^\intercal]^\intercal\in\mathbb{R}^6$ is the generalized velocity of  $\{C\}$, $\vect{\omega}_c\in\mathbb{R}^3$ being its angular velocity, $\vct{H}_c\triangleq\text{diag}(m\vct{I}_3, \mathcal{I}_c)\in\mathbb{R}^{6\times6}$ is the positive definite inertia matrix of the robot, $\vct{g}_c\in\mathbb{R}^6$ the gravity vector and $\vct{C}_c\triangleq\text{diag}(\vct{0}_{3\times3},\vct{S}(\mathcal{I}_c\vect{\omega}_c))\in\mathbb{R}^{6\times6}$ the Coriolis-centrifugal matrix, with $\mathcal{I}_c\triangleq \vct{R}_c\mathcal{I}\vct{R}_c^\intercal\in\mathbb{R}^{3\times3}$ and $\mathcal{I}\in\mathbb{R}^{3\times3}$  the inertia tensor of the main body of the robot. In case in which the $z$-axis of the inertial frame is aligned to the gravity direction, $\vct{g}_c=[ 0\; 0\; m_r g\; 0\; 0\; 0]^\intercal$, with $m_r \in\mathbb{R}^+$ being the mass of the robot and $g\in\mathbb{R}^+$ the constant acceleration due to gravity. 


Consider a torque controlled robot, i.e. accepting joint torque commands $\vect{\tau}_i(t)\in\mathbb{R}^n, i=1,...,4$ (joint torques of the $i$-th leg). For solving the task-space trajectory tracking problem, 
one has to solve \eqref{eq:F_c} with respect to $\vct{F}_a$, i.e. compute the inverse mapping, to calculate the forces that each leg should apply in order to render the commanded force in the task-space,  $\vct{F}_c$, as follows:
\begin{equation}\label{eq:inverse}
\vct{F}_a = \vct{G}^\dagger(\vct{q})\vct{F}_c,
\end{equation}
where $\vct{G}^\dagger\in\mathbb{R}^{12\times6}$ is a right pseudo-inverse of $\vct{G}$. In this point, there are multiple options regarding the pseudo-inverse. Some of them are the right Moore-Penrose pseudo-inverse,  given by $\vct{G}^\dagger\triangleq\vct{G}^\intercal(\vct{G}\vct{G}^\intercal)^{-1}$, which will result in an equal distribution of control effort among the four legs (minimum norm solution), or the right weighted  pseudo-inverse, given by:
\begin{equation}\label{eq:weighted}
\vct{G}^\dagger\triangleq\vct{W}^{-1}\vct{G}^\intercal(\vct{G}\vct{W}^{-1}\vct{G}^\intercal)^{-1},
\end{equation}
with $\vct{W}\in\mathbb{R}^{12\times12}$ being a positive definite weight matrix. The latter will result in distributing the control effort based on the selected weight matrix $\vct{W}$. More specifically, in this case, by selecting a positive definite diagonal matrix $\vct{W}\triangleq\text{diag}(w_{1,1},w_{1,2},w_{1,3},...,w_{4,3})$, the higher the $w_{i,m}$, the less the force appended to the $m$-th direction of the $i$-th leg's tip; for instance a high  value of $w_{3,2}$ as compared to the other $w_{i,m}$-s, will result in appending less force along the  $y$-direction ($m=2$) of the third leg ($i=3$). After computing the corresponding force  in each leg, i.e. $\vct{f}_i$ which is included in $\vct{F}_a$, one can compute the commanded torques from \eqref{eq:joint_jacobian}.

\subsection{The problem of slippage}

From a robot control perspective, slippage of the tip with respect to the supporting surface could potentially result in losing the controllability of the  system. In particular, slippage could lead a) to singular configurations, in which the rank of ${}^c\vct{J}_{i}(q_{i,1},...,q_{i,n})$ will be decreased, b) reaching the joint limits or c) losing contact in one or multiple legs without accounting for it, which is required for  the validity of \eqref{eq:inverse}.  
As long as the force  applied by each leg, i.e. $\vct{f}_i$, is within the friction cone, no slippage of the tip (with respect to the supporting surface) is occurred. More specifically, for a given terrain with constant static friction coefficient $\mu\in\mathbb{R}^+$, the static friction cone (the area in which there is not slippage) is expressed by:
\begin{equation}\label{eq:friction_cone}
\mathcal{C}\triangleq\{\vct{f}_i\in\mathbb{R}^3:\mu|\vct{n}^\intercal\vct{f}_i|>\|(\vct{I}_3 - \vct{n}\vct{n}^\intercal)\vct{f}_i\|\},
\end{equation}
where $\vct{n}\in\mathbb{S}^2$ (with $\mathbb{S}^2\triangleq\{\vct{x}\in\mathbb{R}^3:\|\vct{x}\|=1\}$) is the   normal to the supporting surface vector. 
However, due to the fact that $\mu$ is not easy to measure or estimate  and most of the times it is considered to be unknown, one cannot assess a priori, i.e. before commanding $\vct{f}_i$ to the leg, whether the leg's tip would slip or not.

\subsection{Concept solution} 

Building on our previous work~\cite{Maravgakis2023}, in which we introduced a probabilistic approach for estimating the stable contact probability (i.e. not slipping), we  propose a novel trajectory tracking control scheme incorporating a two-layer on-line adaptation. In particular, in~\cite{Maravgakis2023} the probability of a contact being stable is estimated in real-time, based on a set of IMU sensors mounted on the robot's feet.
This probability is  counter proportional to the probability of slippage. Based on this estimate, we firstly propose an  adaptation law for the weights of distribution of the control effort among all the directions of forces that should be applied by each leg of the robot. The rationale behind the adaptation law is to  append less tangential to the surface forces to the legs for which the slippage probability is high, attracting in this way the appended force towards the friction cone.  
Furthermore,  when the aforementioned force distribution cannot guarantee the elimination of slippage, we propose the dynamic time-scaling of the trajectory (e.g. to slow down the motion), which will consequently yield a reduced control effort magnitude in general.    
The complete adaptive control scheme is shown in Fig.\ref{fig:block}. 

 \begin{figure}[h!]
	\centering
		\includegraphics[width=.45\textwidth]{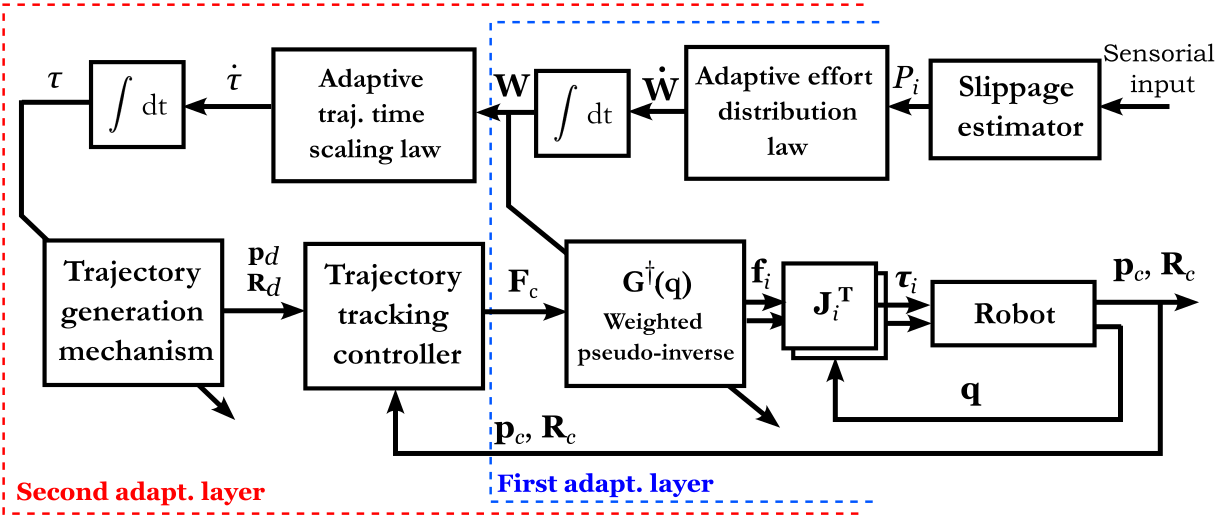}
	    \caption{Block diagram of the proposed adaptive scheme.}
	    \label{fig:block}
\end{figure}

\section{Proposed scheme}

Given a reference trajectory $\vct{p}_d(t)\in\mathbb{R}^3$ an $\vct{R}_d(t)\in SO(3)$ for  frame $\{C\}$ in position and orientation respectively, we consider the trajectory tracking problem, i.e. the problem of minimizing the Euclidean norms of the following errors in time:
\begin{align}
\vct{e}_p&\triangleq\vct{p}_c(t)-\vct{p}_d(t),\\
\vct{e}_o&\triangleq\text{log}\left(\vct{R}_c(t)\vct{R}_d^\intercal(t)\right),
\end{align}
where $\text{log}(\vct{R})\triangleq\vct{k}\theta\in\mathbb{R}^3$ the logarithmic mapping, with $\theta\in[0,\pi]$ being the angle and $\vct{k}\in\mathbb{S}^2$ being the axis of  rotation of a given $\vct{R}$.  The trajectory $\vct{p}_d(t)$, $\vct{R}_d(t)$ could represent the motion of the main robot's body during its locomotion, or even dexterous motions for avoiding collisions in unstructured environments, e.g. the case of passing through a narrow opening. Notice that in case of locomotion one should re-define the matrix $\vct{G}$ on-line, based on the feet that are in contact with the terrain.  

Considering the system dynamics given in \eqref{eq:dyn_model}, the control objective can be achieved by  applying the following state-feedback control law with gravity compensation, representing the commanded generalized force that should be applied to the CoM:
\begin{equation}\label{eq:Command}
\begin{split}
\vct{F}_c \triangleq& 
\vct{H}_c
\begin{bmatrix}
\ddot{\vct{p}}_d \\
\frac{d}{dt}\left( \vct{R}_c\vct{R}_d^\intercal\vect{\omega}_d\right)
\end{bmatrix}
+ 
\vct{C}_c
\begin{bmatrix}
\dot{\vct{p}}_d \\
\vct{R}_c\vct{R}_d^\intercal\vect{\omega}_d
\end{bmatrix} \\
&-
\begin{bmatrix}
k_p \vct{e}_p \\
k_o \vct{e}_o
\end{bmatrix} - \vct{K}_v \vct{e}_v + \vct{g}_c,
\end{split}
\end{equation}
where 
\begin{equation}\label{eq:V_e}
\vct{e}_v\triangleq
\vct{V}_c - 
\begin{bmatrix}
\dot{\vct{p}}_d \\
\vct{R}_c\vct{R}_d^\intercal\vect{\omega}_d
\end{bmatrix},
\end{equation}
$k_p, k_o\in\mathbb{R}^+$, $\vct{K}_v\in\mathbb{R}^{6\times6}$ are constant positive control gains and $\vect{\omega}_d\in\mathbb{R}^3$ is the reference angular velocity which can be calculated by $\vct{S}(\vect{\omega}_d)=\dot{\vct{R}}_d\vct{R}_d^\intercal$. The proof of global asymptotic stability of the origin of the state-space (corresponding to zero error in position and velocity), under the application of the control law \eqref{eq:Command}, is proven in Appendix A.

\subsection{Slippage detection} 
For the slippage detection mechanism, stable contact is considered to be the state in which the robot's foot is in touch with the  ground whilst there is no relative motion between them. To estimate the stable contact probability, a 6D IMU sensor is mounted on each foot of the robot. By exploiting the uncertainty of the inertial measurements, we employ Kernel Density Estimation (KDE) to approximate the Probability Density Function for each axis of the IMU and consequently the per axis stable probability over a small interval, as dictated in \cite{Maravgakis2023}. The method, practically estimates the probability that the inertial measurements are close to zero and finally, since they are independent, the total stable probability is  acquired  via multiplication. Finally, in order to detect when the foot touches the ground, one can utilize force measurements, haptic or dedicated contact sensors. To this end, in this work we exploited the vertical force measurement $(f_{i,z} > 0)$.

\subsection{First layer of adaptation: Adaptive effort distribution}
Based on the above slippage detection mechanism, we propose the following adaptive law for the weights $w_{i,m}, m=1,2$ (the $x-y$ coefficients) of the tangential force directions of the $i$-th leg:
\begin{equation}\label{eq:adapt_law}
\begin{split}
    \dot{w}_{i,1} = \dot{w}_{i,2} \triangleq \alpha P_i, \\
    {w}_{i,1}(0)={w}_{i,2}(0)\triangleq w_0 
    \end{split}
\end{equation}
where $\alpha\in\mathbb{R}^+$ is a tunable constant adaptation gain, $w_0\in\mathbb{R}^+$ the initial value of the weights in $x-y$ direction  and $P_i\in[0,1]$ the probability of slippage of the $i$-th leg. Notice that the normalization of $w_{i,m}, i=1,...,4, m=1,2,3$ is not required, as \eqref{eq:inverse}, \eqref{eq:weighted} do not assume constraints for the values of $\vct{W}$.

\begin{remark}
Notice that \eqref{eq:adapt_law} assumes  the orthogonality between the supporting surface and the gravity direction, for the sake of simplicity of   presentation. However, the generalization to inclined surfaces can be easily done by considering a non-diagonal $\vct{W}$ matrix.   
\end{remark}

\begin{remark}
Given \eqref{eq:adapt_law}, the weights will increase only as long as slippage is estimated, which means that the weights will eventually reach  the value in which the control effort appended to the specific leg does not yield any slippage. The increase of these weights (i.e. the weights corresponding only to the tangential forces) will result in decreasing the magnitude of forces appended towards these directions. Therefore, the appended force $\vct{f}_i$ will converge to the friction cone $\mathcal{C}$, as graphically depicted in Fig.\ref{fig:cone_convergence}.
\end{remark}

 \begin{figure}[h!]
	\centering
		\includegraphics[width=.25\textwidth]{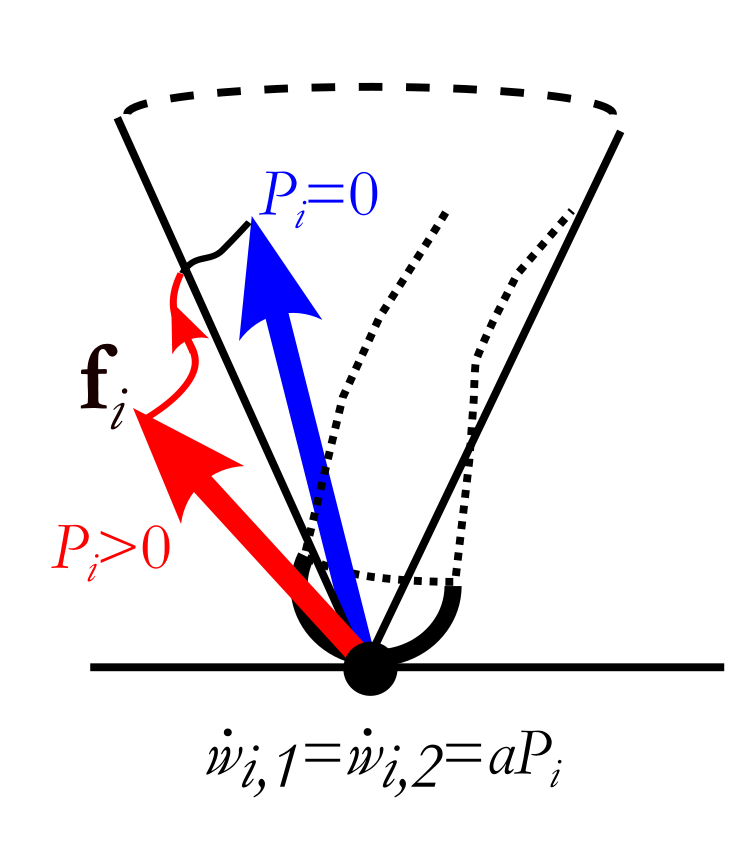}
	    \caption{Convergence of the control effort of each leg towards the friction cone.}
	    \label{fig:cone_convergence}
\end{figure}

\subsection{Second layer of adaptation: Trajectory time-scaling}

The first layer of adaptation could fail 
when all the legs of the robot are contacting a terrain with a relatively low static friction coefficient.
Hence, to handle this type of occasions, we propose the time-scaling of the trajectory, i.e. to sacrifice the temporal accuracy of the task for guaranteeing stability and controllability, maintaining however accuracy with respect to the spatial properties of the path.  

Let $\vct{p}_d(t_v), \vct{R}_d(t_v)$ be the time-parametric trajectory, with $t_v(t)\in\mathbb{R}^+$ being the scaled time parameter and $\beta\in[0,1]$ the time-scaling coefficient. Hence, the evolution of the scaled time parameter is characterized by $\dot{t}_v(t)=\beta(t)$. For instance, setting a constant $\beta=1$ would result in $t_v= t$ and consequently would lead to the execution of the trajectory on a nominal speed, while setting $\beta<1$ would slow down the motion. 
To tackle the problem of global slippage, we propose the utilization of the following time-scaling coefficient:
\begin{equation}\label{eq:time_scaling}
    \beta(t) \triangleq \frac{w_0}{\text{min}(w_{1,1},w_{2,1},w_{3,1},w_{4,1})}.
\end{equation}
The rationale behind \eqref{eq:time_scaling} is to reduce the speed (reflected by $\beta$) 
when slippage has occurred in all four legs, an occasion which is signified by the increase of the weights of all four legs due to \eqref{eq:adapt_law}.
For instance, if at least one of the legs does not face any slippage, then $\text{min}(w_{1,1}(t),w_{2,1}(t),w_{3,1}(t),w_{4,1}(t))$ will be equal to $w_0$ and therefore $\beta$ will be 1, which means that no time scaling would occur.

As the on-line  time scaling is considered, the trajectory should be generated on-line from $t_v$, which is calculated by the integration of  $\dot{t}_v = \beta(t)$ in real-time. Notice that, in such a case $\dot{\vct{p}}_d(t) = \beta(t)\frac{d\vct{p}_d(t_v)}{dt_v}, \dot{\vct{R}}_d(t) =\beta(t)\frac{d\vct{R}_d(t_v)}{dt_v}$.

The complete algorithm of the proposed control scheme is given in Algorithm \ref{algorithm:contr_loop}.

\begin{remark}
    Notice that if the trajectory is generated online by a  dynamical system (e.g. a Dynamic Movement Primitives model \cite{Ijspeert2013}), the application of the aforementioned idea is straight forward, as in that case $\beta$ would correspond to the time scaling parameter of the dynamical system.   
\end{remark}

\begin{algorithm}
	\caption{Implementation of the control loop} \label{algorithm:contr_loop}
	\begin{algorithmic}[1]
		\State Select values for:  $k_p, k_o, \vct{K}_v, w_0$
        \State $\vct{W}:=w_0\vct{I}_{12}, \beta:=1, t_v:=0$ \Comment Initialization
		\While{control is enabled }
			\State Get current state of the robot $\vct{p}_{c}, \vct{R}_{c}, \dot{\vct{p}}_{c}, \vect{\omega}_c$
            \State Estimate  $P_i, \forall i=1,...,4$ \Comment Slip. prob. estimator 
            \State Compute $\dot{w}_{i,1}, \dot{w}_{i,2}, \forall i=1,...,4 $ from \eqref{eq:adapt_law}
            \State Integrate $\dot{w}_{i,1}, \dot{w}_{i,2}, \forall i=1,...,4 $ to update $\vct{W}$
            \State Compute $\beta$ from \eqref{eq:time_scaling}
            \State Integrate $\dot{t}_v=\beta$ to update $t_v$
            \State Compute $\vct{p}_d(t_v), \dot{\vct{p}}_d(t_v), \ddot{\vct{p}}_d(t_v), \vct{R}_d(t_v), \vect{\omega}_d(t_v), \dot{\vect{\omega}}_d(t_v)$
			\State Compute $\vct{F}_c$ from \eqref{eq:Command}
            \State Compute $\vct{F}_a$ (includes the $\vct{f}_i$-s) from \eqref{eq:inverse}, \eqref{eq:weighted}
            \State Compute $\vect{\tau}_i, \forall i=1,...,4$ from \eqref{eq:joint_jacobian}
            \State Command $\vect{\tau}_i, \forall i=1,...,4$ to the joints
        \EndWhile
	\end{algorithmic}
\end{algorithm}

\section{Simulation study}

To assess the performance of the proposed adaptive control scheme we consider three simulation scenarios: a) A simple point-to-point motion to evaluate the trajectory tracking performance, b) a scenario involving the tracking of a periodic motion with the rear right foot  contacting a slippery surface and c) a scenario involving the tracking of a periodic motion with global slippage, i.e. all four legs are contacting a slippery surface. For the simulations, the model of a Unitree Go1 robot is utilized in the Gazebo environment and a control cycle of 2ms is considered. The parameters utilized are $k_p=3000, k_o=150, \vct{K}_v=\text{diag}(550\vct{I}_3,55\vct{I}_3), w_0=35, \alpha=150$. In Fig.\ref{fig:sim_screeshot} the simulation environment is shown, with the yellow area representing the slippery terrain considered in the second scenario (slippage of the rear right foot). 

\begin{figure}[h!]
	\centering
		\includegraphics[width=.3\textwidth]{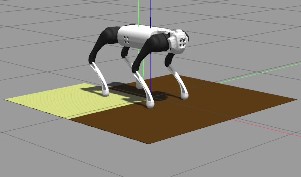}
	    \caption{The initial configuration of the simulations. Yellow area: The slippery area considered for the second scenario.}
	    \label{fig:sim_screeshot}
\end{figure}

\subsection{Scenario 1: Point to point motion}
For this scenario, a terrain with a static friction coefficient of $1.4$ is considered, representing a non-slippery terrain. The desired trajectory is generated online by the following first order dynamical system: $\dot{\vct{p}}_d(t) = \vct{p}_d(t)-\vct{p}_T$, with $\vct{p}_T=\vct{p}_d(0)+[0.1 \; 0.05 \; -0,005]^\intercal$ being the constant target. The initial actual and  desired values are $\vct{p}(0)=[ -0.043\; -0.0037 \; 0.356]^\intercal$m and  $\vct{p}_d(0)=[-0.023 \; 0.0063 \; 0.355]^\intercal$m in order to impose an initial position error of $\vct{e}_p=[-2 \;\; -1 \;\; 0.1]^\intercal$cm. In Fig.\ref{fig:1stscenario_traj} the actual position evolution is compared to the desired trajectory, in which one can notice the tracking performance. Notice that the tracking performance is affected by the unmodelled joint friction that acts as a disturbance to the system, with the $z$-direction being the most disturbed direction, due to the manipulability ellipsoid of the given robot's configuration. One could possibly reduce  this steady state error by further tuning the control gains (as no extensive tuning was performed), or by incorporating an additional integral term to the controller.       

 \begin{figure}[h!]
	\centering
		\includegraphics[width=.4\textwidth]{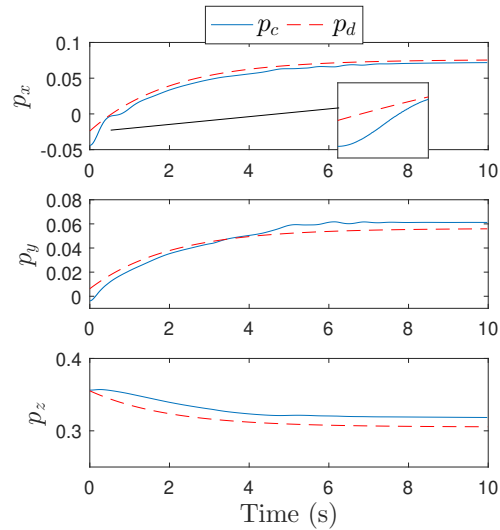}
	    \caption{[Scenario 1: Point-to-point motion] Time evolution of the actual and desired position.}
	    \label{fig:1stscenario_traj}
\end{figure}

\subsection{Scenario 2: One-foot slippage}

For the second scenario,   the rear right foot of the robot ($i=3$) is considered to contact a slippery surface having a static friction coefficient of $0.4$, which is considered to be unknown for the controller. For the rest of the feet a non-slippery surface is considered.  For comparison,  two tests are performed, namely one with the adaptive mechanism  and the other without it. The desired trajectory involves a periodic sinusoidal motion, executing an ellipse on the $x-z$ plane, for position and a periodic rotation around the $x$-axis for orientation. The frequency of the periodic trajectory is $0.7$Hz  and $0.2$Hz in position and orientation respectively. The  weights of distribution along the $x$ direction of each leg (which is equal to the ones along the $y$ direction), i.e. $w_{i,1}$, are depicted in Fig.\ref{fig:2ndscenario_w}, alongside with the slippage probability provided by the estimator, i.e. $1-P_i$. Notice the rise of the  value of $w_{3,1}$ (the leg that slips), which results in appending less force along the  $x-y$ directions of the third leg. Further notice that the third leg stops slipping after the adaptation which means that the force converged to a value within the friction cone and the system reaches a stable steady-state condition. In Fig.\ref{fig:2ndscenario_errors}, the position and orientation errors are depicted, with and without the proposed adaptive scheme for comparison purposes. Notice that without the proposed adaptation mechanism, the system is not able to maintain its stability, as the robot loses contact with the environment at $t\approx4.5$s. Last, notice that $\beta=1$ during the whole simulation, due to equation \eqref{eq:time_scaling} and the  fact that $w_{i,1}=35,\forall i=1,2,4$, which means that the  first adaptation layer can sufficiently provide a solution by dynamically distributing the control effort.

 \begin{figure}[h!]
	\centering
		\includegraphics[width=.45\textwidth]{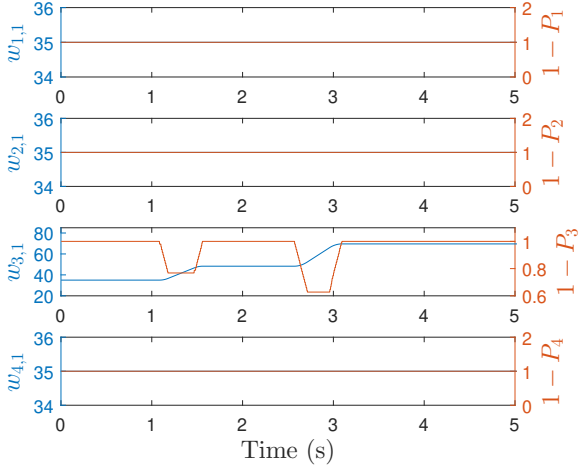}
	    \caption{[Scenario 2: One-foot slippage] Weight adaptation due to the first layer (the second layer is not enabled).}
	    \label{fig:2ndscenario_w}
\end{figure}

 \begin{figure}[h!]
	\centering
		\includegraphics[width=.41\textwidth]{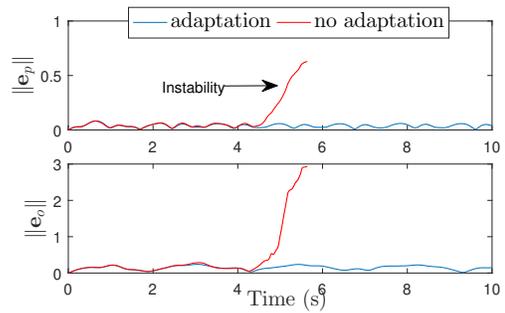}
	    \caption{[Scenario 2: One-foot slippage] Position and orientation error norms with and without adaptation.}
	    \label{fig:2ndscenario_errors}
\end{figure}

\subsection{Scenario 3: Global slippage}

For this scenario,  all four  legs of the robot are considered to contact the slippery surface having a static friction coefficient of $0.4$. For comparison, we performed two tests, namely one with the adaptive control scheme and one without it and the same trajectory with that of the second scenario is considered. The  weights of distribution along the $x$ direction of each leg (which is equal to the ones along the $y$ direction), i.e. $w_{i,1}$, as well as  the time-scaling parameter $\beta(t)$ are depicted in Fig.\ref{fig:3ndscenario_w}, alongside with the slippage probability provided from the estimator, i.e. $1-P_i$. Notice the rise of the values of all $w_{i,1}$, $i=1,...,4$, which results in slowing down the motion, which is reflected by the reduction of $\beta(t)$ which converges to the value of $\beta \approx 0.77$ after $t\approx5$s. In Fig.\ref{fig:3ndscenario_traj}, the evolution of the position of the CoM in time is depicted both with and without the proposed control scheme.  
Notice that without the proposed adaptation mechanism, the system is, also in this case, not able to maintain the stability of the system as the robot, signified by the drop of the CoM in Fig.\ref{fig:3ndscenario_traj}. Last, notice the smooth on-line time-scaling of the trajectory occurred after $t\approx5$s.

 \begin{figure}[h!]
	\centering
		\includegraphics[width=.45\textwidth]{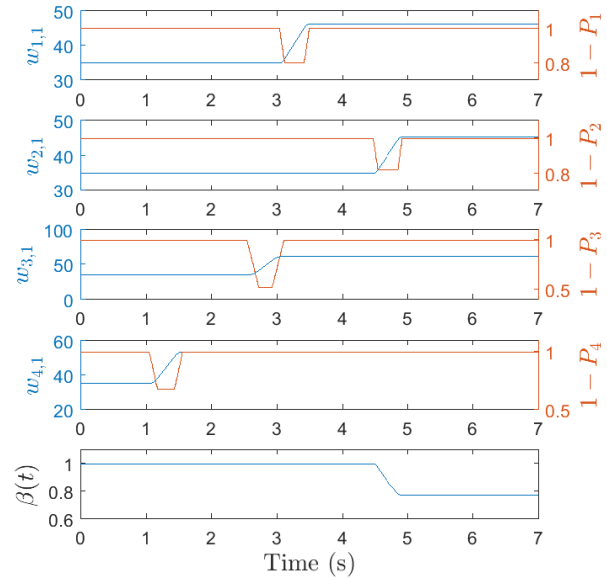}
	    \caption{[Scenario 3: Global slippage] Weight adaptation due to the first and second layer.}
	    \label{fig:3ndscenario_w}
\end{figure}

 \begin{figure}[h!]
	\centering
		\includegraphics[width=.45\textwidth]{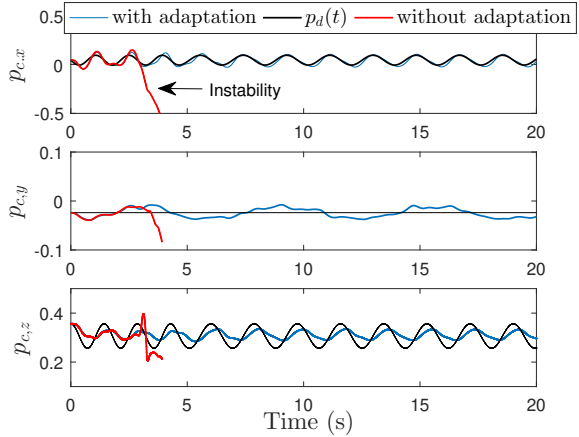}
	    \caption{[Scenario 3: Global slippage] Evolution of the position of the CoM in time, with and without adaptation.}
	    \label{fig:3ndscenario_traj}
\end{figure}

\section{Experimental validation}

The implementation of the is done utilizing  a real Unitree Go1 robot, to validate the adaptation performed by the first layer of the adaptation mechanism. In particular, a 6DOF IMU is attached to the second leg of the robot, as shown in Fig.\ref{fig:exp_screenshot}, which is in contact with a slippery surface, (i.e. lubricant is utilized to emulate the slippery area below the second leg), while the pose of the robot is found on-line via an external camera with an off-the-shelf visual odometry system\footnote{https://github.com/IntelRealSense/librealsense/blob/master/doc/t265.md} and therefore initial robot's pose is considered as the world frame for the experiment. 
The proposed adaptive scheme parameters are set to   $k_p=2400, k_o=, \vct{K}_v=\text{diag}(280\vct{I}_3,28\vct{I}_3), w_0=35, \alpha=1000$. The robot was commanded to move along the $x$-axis with a similar to the second simulation periodic trajectory for the axis of motion, having a frequency of $0.4$Hz. Fig.\ref{fig:exp_screenshot} depicts the experimental setup with the robot being in the initial configuration, while in Fig. \ref{fig:experiment_w}, the weight corresponding to the $x-y$ directions of the second leg is given, alongside with the slippage probability estimate; the weights of the rest of the legs remained  unaltered during the experiment. In Fig.\ref{fig:exp_evo_pos} the  evolution of position in time is depicted utilizing the adaptive scheme and without its utilization, for comparison. Notice that the activation of the first adaptation layer  results in maintaining stability, while when executing the same scenario without enabling the adaptation mechanism the robot is not able to maintain stability at $t\approx14.6$s. Further, notice that without the adaptation mechanism the tracking performance is affected by the slippage of the second leg, as it triggers unmodelled dynamics.

 \begin{figure}[h!]
	\centering
		\includegraphics[width=.35\textwidth]{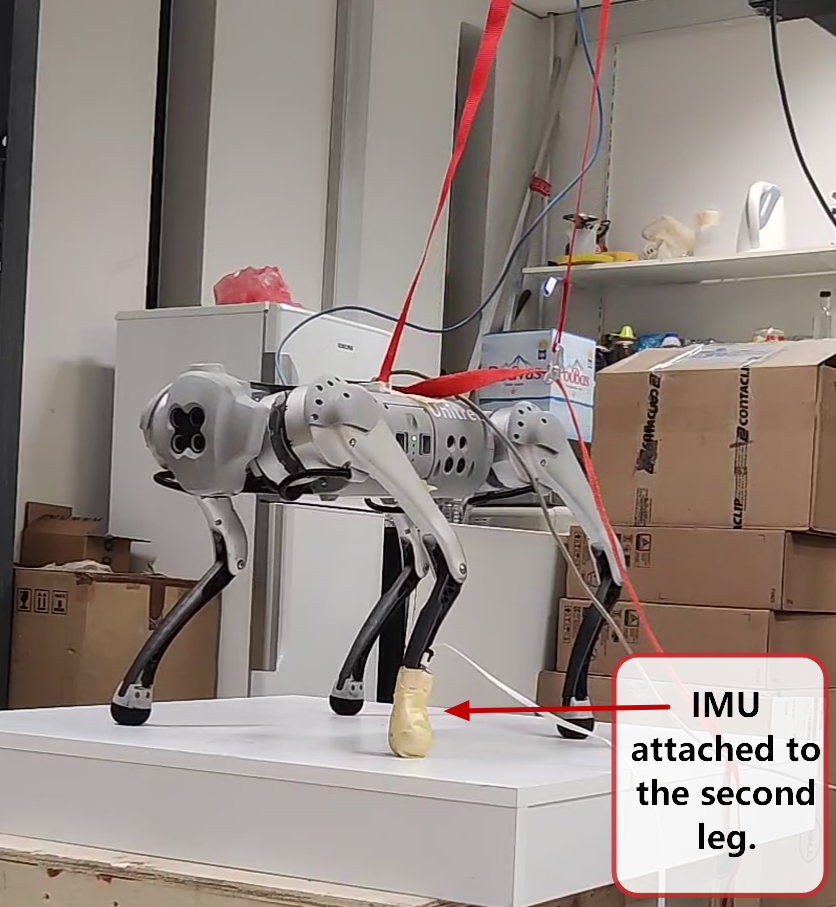}
	    \caption{The experimental setup and initial configuration of the robot.}
	    \label{fig:exp_screenshot}
\end{figure}

 \begin{figure}[h!]
	\centering
		\includegraphics[width=.45\textwidth]{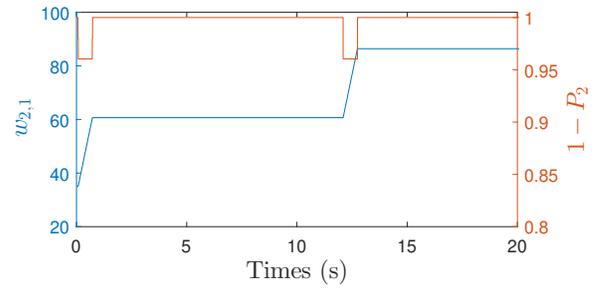}
	    \caption{The second weight during the experimental validation.}
	    \label{fig:experiment_w}
\end{figure}

 \begin{figure}[h!]
	\centering
		\includegraphics[width=.45\textwidth]{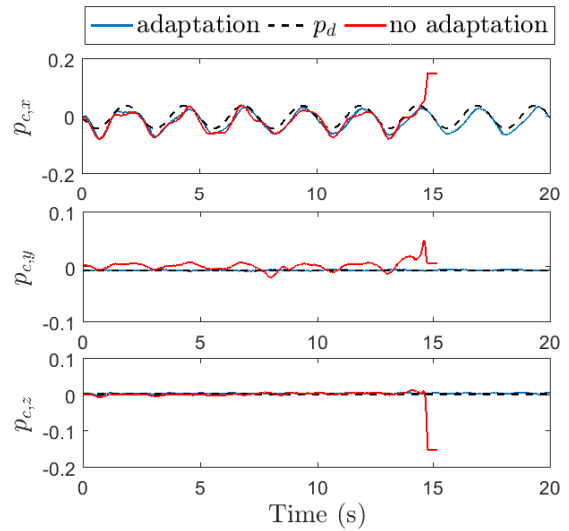}
	    \caption{ Evolution of the position of the CoM in time, with and without adaptation.}
	    \label{fig:exp_evo_pos}
\end{figure}

\section{Conclusions and future work}
 In this work, an adaptive trajectory tracking controller is proposed for quadruped robots, which involves two prioritized  layers of adaptation for minimizing the slippage of one or multiple legs. The first adaptation layer considers the dynamic distribution of the control effort among the legs, given the slippage probability for each leg. The second layer, which is enabled only if the problem cannot be solved by the dynamic distribution of the effort, which may occur when all for legs slip, acts on the time-scaling of the trajectory by dynamically and smoothly slowing down the motion, without affecting the spatial properties of the task. The proposed method is proven to be asymptotically stable. Furthermore, it is shown through simulations and experiments  that the method equips the system with robustness, as it is able to minimize the slippage of the legs and it ensures the stability and controllability of the robot. 
 This is accomplished without sacrificing the task space trajectory, this feature is particularly useful in both visual and depth based Simultaneous Localization and Mapping applications where stable and precise movement is essential for the performance of the approach.  
 Future work will mainly focus on employing the proposed adaptation control scheme coupled with our contact estimation approach as the foundation for the development of a model-based dexterous and slippery-robust dynamic locomotion algorithm. By capitalizing on the real-time adjustment of the weights for each leg we aim at developing algorithms for walking and running in terrain agnostic environments while maintaining the task space desired trajectory.

\section*{Appendix A}
After substituting \eqref{eq:Command} in \eqref{eq:dyn_model}, we get the following closed loop system dynamics:

\begin{align}\label{eq:dyn_model_proof}
\vct{H}_c\dot{\vct{e}}_v 
&= 
-(\vct{C}_c+\vct{K}_v)\vct{e}_v -
\begin{bmatrix}
k_p \vct{e}_p \\
k_o \vct{e}_o
\end{bmatrix}, \\ \label{eq:dyn_model_proof_2}
\begin{bmatrix}
\dot{\vct{e}}_p  \\
\dot{\vct{e}}_o 
\end{bmatrix}
&=
\begin{bmatrix}
 \vct{I}_3 & \vct{0}_3\\
\vct{0}_3 & \vct{J}_l(\vct{e}_o) 
\end{bmatrix}
\vct{e}_v, 
\end{align}
where $\vct{J}_l(\vct{e}_o)\in\mathbb{R}^{3\times3}$ the matrix mapping the orientation part of $\vct{e}_v$ to $\dot{\vct{e}}_o$, as detailed in \cite{Koutras2021}, for which the following holds:  $\vct{J}_l^\intercal\vct{e}_o=\vct{J}_l\vct{e}_o=\vct{e}_o$ (as shown in \cite{Koutras2021}).
\begin{theorem} The origin of the state-space of the  system \eqref{eq:dyn_model_proof}, i.e. $(\vct{e}_p,\vct{e}_o,\vct{e}_v)=(\vct{0},\vct{0},\vct{0})$, is globally asymptotically stable. 
\end{theorem}

\begin{proof}
Consider the following candidate Lyapunov function:
\begin{equation}
L = \frac{k_p}{2}\|\vct{e}_p\|^2 +\frac{k_o}{2}\|\vct{e}_o\|^2 + \frac{1}{2}\vct{e}_v^\intercal\vct{H}_c\vct{e}_v.
\end{equation}
By taking its time derivative, we get:
\begin{equation}\label{eq:Lyap_dot}
\dot{L} = k_p\vct{e}_p^\intercal\dot{\vct{e}}_p +k_o\vct{e}_o^\intercal\dot{\vct{e}}_o + \frac{1}{2}\vct{e}_v^\intercal\dot{\vct{H}}_c\vct{e}_v + 
\vct{e}_v^\intercal\vct{H}_c\dot{\vct{e}}_v.
\end{equation}
After substituting $\vct{H}_c\dot{\vct{e}}_v$ from \eqref{eq:dyn_model_proof} to \eqref{eq:Lyap_dot} and utilizing the skew symmetric property, i.e. $\vct{e}_v^\intercal\left(\dot{\vct{H}}_c-2\vct{C}_c\right)\vct{e}_v=0$,  we get:
\begin{equation}\label{eq:Lyap_dot2}
\dot{L} = k_p\vct{e}_p^\intercal\dot{\vct{e}}_p +k_o\vct{e}_o^\intercal\dot{\vct{e}}_o  + \vct{e}_v^\intercal\left(-\vct{K}_v\vct{e}_v - \begin{bmatrix}
k_p \vct{e}_p \\
k_o \vct{e}_o
\end{bmatrix}\right).
\end{equation}
By utilizing \eqref{eq:dyn_model_proof_2} and the property  $\vct{J}_l^\intercal\vct{e}_o=\vct{J}_l\vct{e}_o=\vct{e}_o$, \eqref{eq:Lyap_dot2} becomes:
\begin{equation}\label{eq:Lyap_dot3}
\dot{L} =  -\vct{K}_v\vct{e}_v^\intercal\vct{e}_v,
\end{equation}
which is less or equal to zero for all $\vct{e}_p\neq\vct{0}, \vct{e}_o\neq\vct{0}, \vct{e}_v\neq\vct{0}$. Hence, by invoking the LaSalle theorem we can conclude that the origin is globally asymptotically stable. 
\end{proof}

\bibliographystyle{ieeetr}
\bibliography{mybib}

\end{document}